\documentclass[twoside,11pt]{article}
\usepackage{jmlr2e} 

\usepackage{fancyhdr}
\usepackage{epsfig,graphicx}
\usepackage[T1]{fontenc} 
\usepackage{epsf} 
\usepackage{color,etoolbox}
\usepackage{amsmath,amsfonts,amssymb,bm,dsfont, mathtools}
\usepackage{float}

\usepackage{url}
\usepackage[inline]{enumitem}
\usepackage{graphicx} 
\usepackage{pdfpages}
\usepackage{subfigure} 
\usepackage{verbatim}
\usepackage{booktabs}

\usepackage{algorithm}
\usepackage{algpseudocode}
\usepackage{wrapfig}
\usepackage{lipsum}

\newcommand{\newref}[2][]{\hyperref[#2]{#1~\ref*{#2}}}
\renewcommand{\eqref}[1]{\hyperref[#1]{(\ref*{#1})}}

\newcommand{\vect}[1]{\ensuremath{\mathbf{#1}}}

\newcommand{\argmin}{\mathop{\rm argmin}}

\newcommand{\iprod}[2]{\langle #1, #2 \rangle}

\newcommand{\abs}[1]{\left|{#1}\right|}
\newcommand{\norm}[1]{\|{#1} \|}

\newcommand{\oracle}[2]{\mathcal{O}_{#1}\left(#2\right)}

\newcommand{\E}[1]{\mathbb{E}\left[#1\right]}
\newcommand{\Eover}[2]{\mathbb{E}_{#1}\left[#2\right]}

\renewcommand{\Pr}{\mathbb{P}}

\newcommand{\e}{\vect{e}}

\newcommand{\x}{\vect{x}}
\newcommand{\bx}{\bar{\vect{x}}}

\newcommand{\y}{\vect{y}}

\ShortHeadings{Online Non-Convex Learning}{Suggala and Netrapalli}
\firstpageno{1}

\begin{document}


\title{Online Non-Convex Learning: Following the Perturbed Leader is Optimal}

\author{\name Arun Sai Suggala \email asuggala@cs.cmu.edu \\
       \addr Machine Learning Department\\
       Carnegie Mellon University
       \AND
       \name Praneeth Netrapalli \email praneeth@microsoft.com \\
       \addr Microsoft Research, India}

\editor{}

\maketitle

\begin{abstract}
We study the problem of online learning with non-convex losses, where the learner has access to an offline optimization oracle. We show that the classical Follow the Perturbed Leader (FTPL) algorithm achieves optimal regret rate of $O(T^{-1/2})$ in this setting. This improves upon the previous best-known regret rate of $O(T^{-1/3})$ for FTPL. We further show that an optimistic variant of FTPL achieves better regret bounds when the sequence of losses encountered by the learner is ``predictable''.
\end{abstract}

\section{Introduction}
\label{sec:intro}
In this work, we study the problem of online learning with non-convex losses, where, in each iteration, the learner chooses an action and observes a loss which could potentially be non-convex. The goal of the learner is to choose a sequence of actions which minimize the cumulative loss suffered over the course of learning. The paradigm of online learning has been studied in a number of fields, including game theory, machine learning, statistics and has several practical applications. In recent years a number of efficient algorithms have been developed for online learning. Convexity of the loss functions has played a central role in the development of many of these techniques. In this work, we consider a more general setting, where the sequence of loss functions encountered by the learner could be non-convex. Such a setting has numerous applications in machine learning, especially in adversarial training~\citep{szegedy2013intriguing}, robust optimization and training of Generative Adversarial Networks (GANs)~\citep{goodfellow2014generative}.

As mentioned above, most of the existing works on online optimization have focused on convex loss functions \citep{hazan2016introduction}. A number of computationally efficient approaches have been proposed for regret minimization in this setting. However, when the losses are non-convex, minimizing the regret is computationally hard. Recent works on learning with non-convex losses get over this computational barrier by either working with a restricted class of loss functions such as approximately convex losses~\citep{gao2018online} or by optimizing a computationally tractable notion of regret~\citep{hazan2017efficient}. Consequently, the techniques studied in these papers do not guarantee vanishing regret for general non-convex losses. Another class of approaches consider general non-convex losses, but assume access to a sampling oracle~\citep{maillard2010online, krichene2015hedge} or an offline optimization oracle~\citep{gonen2018learning}. Of these, assuming access to an offline optimization oracle is reasonable, given that in practice, simple heuristics such as stochastic gradient descent seem to be able to find approximate global optima reasonably fast even for complicated tasks such as training deep neural networks.

In a recent work \citet{gonen2018learning} take this later approach, where they assume access to an offline optimization oracle, and show that the classical Follow the Perturbed Leader (FTPL) algorithm achieves $O(T^{-1/3})$ regret for general non-convex losses which are Lipschitz continuous. In this work, we improve upon this result and show that FTPL in fact achieves optimal $O(T^{-1/2})$ regret.

\section{Problem Setup and Main Results}
\label{sec:setup}
Let $\mathcal{X} \subseteq \mathbb{R}^d$ denote the set of all possible moves of the learner. 
In the online learning framework, on each round $t$, the learner makes a prediction $\x_t \in \mathcal{X}$ and the nature/adversary simultaneously chooses a loss function $f_t:\mathcal{X} \rightarrow \mathbb{R}$ and observe each others actions. The goal of the learner is to choose a sequence of actions $\{\x_t\}_{t=1}^T$ such that the following notion of regret is small
\[
\frac{1}{T}\sum_{t = 1}^T f_t(\x_t) - \frac{1}{T}\inf_{\x \in \mathcal{X}}\sum_{t=1}^Tf_t(\x).
\]
In this work we assume that $\mathcal{X}$ is bounded and has $\ell_{\infty}$ diameter of $D$, which is defined as \mbox{$D =  \sup_{\x,\y \in \mathcal{X}}\|\x-\y\|_{\infty}$.}
Moreover, we assume that the sequence of loss functions $f_t$ chosen by the adversary are L-Lipschitz with respect to $\ell_{1}$ norm, that is, for all $\x,\y \in \mathcal{X},$ \mbox{$\ |f_t(\x)-f_t(\y)| \leq L\|\x-\y\|_1.$} 
\paragraph{Approximate Optimization Oracle.} Our results rely on an offline optimization oracle which takes as input a function $f:\mathcal{X}\rightarrow \mathbb{R}$ and a $d$-dimensional vector $\sigma$ and returns an approximate minimizer of $f(\x)-\iprod{\sigma}{\x}$. An optimization oracle is called ``$(\alpha, \beta)$-approximate optimization oracle'' if it returns $\x^*\in \mathcal{X}$ such that 
\[
f(\x^*) - \iprod{\sigma}{\x^*} \leq \inf_{\x \in \mathcal{X}} f(\x) - \iprod{\sigma}{\x} + \left(\alpha + \beta \|\sigma\|_{1}\right),
\]
We denote such an optimization oracle with $\oracle{\alpha, \beta}{f-\sigma}$. 
\paragraph{FTPL.} Given access to an $(\alpha, \beta)$-approximate offline optimization oracle, we study the FTPL algorithm which is described by the following prediction rule (see Algorithm~\ref{alg:ftpl}).
\begin{equation}
\label{eqn:ftpl_pred}
\x_t = \oracle{\alpha, \beta}{\sum_{i = 1}^{t-1}f_i-\sigma_t},
\end{equation}
where $\sigma_t \in \mathbb{R}^d$ is a random perturbation such that $\sigma_{t,j}$, the $j^{th}$ coordiante of $\sigma_t$, is sampled from $\text{Exp}(\eta)$, the exponential distribution with parameter $\eta$\footnote{Recall, $Z$ is an exponential random variable with parameter $\eta$ if  $P(Z\geq s) =\exp(-\eta s)$}.
\begin{algorithm}[tbh]
\caption{Follow the Perturbed Leader (FTPL)}
\label{alg:ftpl}
\begin{algorithmic}[1]
  \small
  \State \textbf{Input:}   Parameter of exponential distribution $\eta$, approximate optimization oracle $\mathcal{O}_{\alpha, \beta}$
  \For{$t = 1 \dots T$}
  \State Generate random vector $\sigma_t$ such that $\{\sigma_{t,j}\}_{j = 1}^d \stackrel{i.i.d}{\sim} \text{Exp}(\eta)$
  \State Predict $\x_t$ as
  \[
  \x_t = \oracle{\alpha, \beta}{\sum_{i = 1}^{t-1}f_i-\sigma_t}.
  \]
  \State Observe loss function $f_t$
  \EndFor
\end{algorithmic}
\end{algorithm}
\paragraph{Optimistic FTPL (OFTPL).} In the general online learning setting considered above, we assumed that the loss functions could possibly be chosen in an adversarial manner by nature. However, in certain applications, the loss functions may not be adversarial. Instead, they might have some patterns and could be predictable. In such cases, \citet{rakhlin2012online} present algorithms for online linear optimization which can exploit the predictability of losses to obtain better regret bounds. We show that the techniques of \citet{rakhlin2012online} can be extended to the online non-convex optimization setting considered in this work.

Let $g_t[f_1\dots f_{t-1}]$ be our guess of the loss $f_t$ at the beginning of round $t$, with $g_1 = 0$.  To simplify the notation, in the sequel, we suppress the dependence of $g_t$ on $\{f_{i}\}_{i=1}^{t-1}$. Some potential choices for $g_t$ that could be of interest are $f_{t-1}$, $\frac{1}{t-1}\sum_{i=1}^{t-1} f_i$. For a thorough discussion on the choices of $g_t$ and concrete examples where predictable loss functions arise, we refer the reader to ~\citet{rakhlin2012online,rakhlin2013optimization}. Given $g_t$, we predict $\x_t$ in OFTPL as
\begin{equation}
\label{eqn:oftpl_pred}
\x_t =\oracle{\alpha, \beta}{\sum_{i = 1}^{t-1}f_i + g_t-\sigma_t}
\end{equation}
When our guess $g_t$ is close to $f_t$ we expect OFTPL to have a smaller regret. In Theorem~\ref{thm:oftpl} we show that the regret of OFTPL depends only on $(f_t-g_t)$. 
\subsection{Main Results}
We present our main results for an oblivious adversary who fixes the sequence of losses $\{f_t\}_{t=1}^T$ ahead of the game. Following \cite{hutter2005adaptive, cesa2006prediction}, one can show that any algorithm that is guaranteed to work against an oblivious adversary  also works for a non-oblivious adversary, whose actions are allowed to depend on the past predictions of the algorithm. For the sake of completeness, we present a proof of this reduction from non-oblivious to oblivious adversary model in Appendix~\ref{sec:apx_reduction}. 
\begin{theorem}[Non-Convex FTPL]
\label{thm:ftpl}
Let $D$ be the $\ell_{\infty}$ diameter of $\mathcal{X}$. Suppose the losses encountered by the learner are $L$-Lipschitz w.r.t $\ell_{1}$ norm. Moreover, suppose the optimization oracle used by Algorithm~\ref{alg:ftpl}  is a ``$(\alpha, \beta)$-approximate'' optimization oracle. For any fixed $\eta$, the predictions of Algorithm~\ref{alg:ftpl} satisfy the following regret bound
\[
\Eover{}{\frac{1}{T}\sum_{t = 1}^T f_t(\x_t) - \frac{1}{T}\inf_{\x \in \mathcal{X}}\sum_{t=1}^Tf_t(\x)} \leq O\left(\eta d^2D L^2  + \frac{d(\beta T + D)}{\eta T} + \alpha + \beta d L\right).
\]
\end{theorem}

\begin{theorem}[Non-Convex OFTPL]
\label{thm:oftpl}
Let $D$ be the $\ell_{\infty}$ diameter of $\mathcal{X}$. Suppose our guess $g_t$ is such that $(g_t-f_t)$ is $L_t$-Lipschitz w.r.t $\ell_{1}$ norm, for all $t \in [T]$. For any fixed $\eta$, OFTPL with access to a ``$(\alpha, \beta)$-approximate'' optimization oracle satisfies the following regret bound
\[
\Eover{}{\frac{1}{T}\sum_{t = 1}^T f_t(\x_t) - \frac{1}{T}\inf_{\x \in \mathcal{X}}\sum_{t=1}^Tf_t(\x)} \leq O\left(\eta d^2D \sum_{t = 1}^T\frac{L_t^2}{T} + \frac{d(\beta T + D)}{\eta T} +\alpha + \beta d \sum_{t=1}^T\frac{L_t}{T}\right).
\]
\end{theorem}
\noindent The above result shows that for appropriate choice of $\eta$, FTPL achieves \mbox{$O(d^{\frac{3}{2}}T^{-\frac{1}{2}} + \alpha + \beta d^{\frac{3}{2}} T^{\frac{1}{2}})$} regret. This also shows that when $\alpha = O(T^{-\frac{1}{2}}), \beta = O(T^{-1})$, FTPL achieves the optimal $O(T^{-\frac{1}{2}})$ regret. This improves upon the $O(T^{-\frac{1}{3}})$ regret bound obtained by~\citet{gonen2018learning}.
We note that the above results can be generalized to infinite-dimensional spaces such as $\ell^1$ space of sequences. To do this we assume that the domain $\mathcal{X}$ is bounded and can be enclosed in a hyper-rectangle with edge length $D_i$ along the $i^{th}$ standard basis vector. Through a more careful analysis we can obtain regret bounds that depend on  the \emph{effective dimension} of $\mathcal{X}$, which is defined as $\frac{\sum_{i=1}^{d}  D_i}{\max_{i} D_i}$, instead of $d$. 

Before we conclude the section we point out that as an immediate consequence of the above regret bounds, we obtain algorithms for approximating the mixed strategy Nash equilibria of general non-convex non-concave saddle point problems of the form $\displaystyle \min_{\x \in \mathcal{X}} \max_{\y \in \mathcal{Y}} M(\x, \y)$. This follows from the observation that saddle point problems can be solved by playing two online optimization algorithms against each other~\citep{cesa2006prediction, hazan2016introduction}.



\section{Background}
\label{sec:background}
In this section we briefly review the relevant literature on online learning in both convex and non-convex settings.
\paragraph{Online Convex Optimization.} When the domain $\mathcal{X}$ and the loss functions $f_t$ encountered by the learner are convex, a number of efficient algorithms for regret minimization have been studied.  Most of these algorithms fall into three broad categories, namely Follow the Regularized Leader (FTRL), Online Mirror Descent (OMD)~\citep{hazan2016introduction} and Follow the Perturbed Leader (FTPL)~\citep{kalai2016efficient}. FTRL algorithms make a prediction in each iteration by minimizing  $\argmin_{\x} \sum_{i=1}^{t-1} f_i(\x) + R(\x)$, where $R$ is a strongly convex regularizer. The regularization $R$ plays a crucial role in the performance of the algorithm and helps avoid overfitting to the observed loss functions. Similar to FTRL, OMD also relies on explicit regularization to guarantee vanishing regret. In fact, under certain settings, both OMD and FTRL algorithms are known to be equivalent~\citep{mcmahan2011follow}. For a broad class of online convex optimization problems, FTRL and OMD are known to achieve optimal regret guarantees.

FTPL algorithms rely on random perturbation of loss functions to guarantee vanishing regret. This random perturbation can be viewed as having a similar role as the explicit regularization used in FTRL and OMD. In a recent work \citet{abernethy2016perturbation} use duality to connect FTPL and FTRL. They show that every instance of FTPL is also an instance of FTRL.

\paragraph{Online Non-Convex Optimization.} A natural question that arises in the context of online non-convex learning is whether there exist counterparts of FTRL and OMD which achieve vanishing regret. Unfortunately, the answer is no. As we show in the following Proposition, there exists no deterministic algorithm that can achieve vanishing regret when the losses are non-convex.
\begin{proposition}
\label{prop:no_deterministic}
No deterministic algorithm can achieve $o(1)$ regret in the setting of online non-convex learning. 
\end{proposition}
The above Proposition shows that only randomized algorithms can achieve vanishing regret. Recent works of \citet{maillard2010online, krichene2015hedge} consider the natural extension of Exponential Weight Algorithm to continuous domains and show that the resulting algorithm has vanishing regret in the setting of online non-convex learning. The algorithms studied in these works rely on an offline sampling oracle which can generate samples from any given probability distribution. In another line of work, ~\citet{gonen2018learning} study the classical FTPL algorithm with access to a certain offline optimization oracle and show that it achieves $O(T^{-1/3})$ regret. As an immediate consequence of this result, the authors show that both online adversarial learning model and statistical learning model are computationally equivalent.

\section{Non-Convex FTPL}
\label{sec:ftpl}
In this section, we present a proof of Theorem~\ref{thm:ftpl}. Since we are in the oblivious adversary setting, it suffices to work with a single random vector $\sigma$, instead of generating a new random vector in each iteration.
The first step in the proof involves relating the expected regret to the stability of prediction, which is a standard step in the analysis of many online learning algorithms. 
\begin{lemma}
    \label{lem:reg_to_stbl}
     The regret of Algorithm~\ref{alg:ftpl} can be upper bounded as 
    \begin{equation}
        \label{eqn:stability}
        \E{\sum_{t = 1}^T f_t(\x_t) - \inf_{\x \in \mathcal{X}}\sum_{t=1}^Tf_t(\x)} \leq L\sum_{t=1}^T\underbrace{\E{\norm{\x_t-\x_{t+1}}_1}}_{Stability} +  \frac{d(\beta T+D)}{\eta} + \alpha T.
    \end{equation}
\end{lemma} 
In the rest of the proof we focus on bounding the stability term $\E{\norm{\x_t-\x_{t+1}}_1}$. The randomness used in the algorithm is crucial for bounding its stability. The more randomness we add, the more stable the algorithm is. However, there is a price we pay for adding randomness. It causes the algorithm to make poor predictions, which leads to worse regret. This is evident in the second term in the upper bound in Equation~\eqref{eqn:stability}, which increases as $\eta$ decreases. 

We first provide an brief sketch of the proof in the $1$-dimensional case. Similar to the proof of \citet{gonen2018learning}, our proof relies on showing certain monotonicity properties of the predictions of the algorithm. Letting $\x_t(\sigma)$ be the prediction in the $t^{th}$ iteration of FTPL with random perturbation $\sigma$, we show that the predictions are monotonic functions of $\sigma$
\[
\forall t,c > 0, \quad \x_t(\sigma + c) \geq \x_t(\sigma).
\]
Moreover, we show that
\[
\forall c > L,\quad \min\left\lbrace\x_t(\sigma+c), \x_{t+1}(\sigma+c)\right\rbrace \geq \max\left\lbrace\x_t(\sigma), \x_{t+1}(\sigma)\right\rbrace.
\]
Since the domain is bounded, these two properties imply that the functions $\x_t(\sigma), \x_{t+1}(\sigma)$ should be close to each other for sufficiently large values of $\sigma$ (see Figure~\ref{fig:monotonicity_props} for an illustration). The closeness of these two functions immediately implies the stability of the algorithm. In what follows, we formalize this argument and extend it to the high-dimensional case.
\begin{figure}[h]
    \centering
    \includegraphics[scale=0.5, trim={0 0 0 0cm},clip]{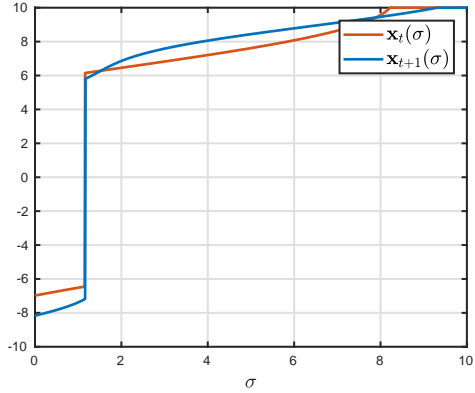}
    \caption{Illustration of monotonicity properties of the predictions of FTPL on a $1$-dimensional example with $D=10, L=2$.}
    \label{fig:monotonicity_props}
\end{figure}
\begin{lemma}[Monotonicity 1]
\label{lem:monotone1}
 Let $\x_t(\sigma)$ be the prediction of FTPL in iteration $t$, with random perturbation $\sigma$. Let $\e_i$ denote the $i^{th}$ standard basis vector and $\x_{t,i}$ denote the $i^{th}$ coordinate of $\x_t$. Then the following monotonicity property holds for any $c> 0$
\[
\x_{t,i}(\sigma+c\e_i) \geq \x_{t,i}(\sigma) - \frac{2(\alpha+\beta\|\sigma\|_{1})}{c} - \beta.
\]
\end{lemma}
\begin{proof}
Let $f_{1:t}(\x) = \sum_{i = 1}^tf_i(\x)$ and $\sigma' = \sigma+c \e_i$. Moreover, let $\gamma(\sigma) = \alpha + \beta \|\sigma\|_{1}$ be the approximation error of the offline optimization oracle. From the approximate optimality of $\x_t(\sigma)$ we have
\begin{align*}
&f_{1:t-1}(\x_t(\sigma)) - \iprod{\sigma}{\x_t(\sigma)}\\
&\quad \leq f_{1:t-1}(\x_t(\sigma')) - \iprod{\sigma}{\x_t(\sigma')} + \gamma(\sigma)\\
&\quad = f_{1:t-1}(\x_t(\sigma')) - \iprod{\sigma'}{\x_t(\sigma')} + c\x_{t,i}(\sigma') + \gamma(\sigma)\\
&\quad \stackrel{(a)}{\leq} f_{1:t-1}(\x_t(\sigma)) - \iprod{\sigma'}{\x_t(\sigma)} + c\x_{t,i}(\sigma') + \gamma(\sigma) + \gamma(\sigma')\\
&\quad = f_{1:t-1}(\x_t(\sigma)) - \iprod{\sigma}{\x_t(\sigma)} + c\left(\x_{t,i}(\sigma') - \x_{t,i}(\sigma)\right) + \gamma(\sigma) + \gamma(\sigma'),
\end{align*}
where $(a)$ follows from the approximate optimality of $\x_t(\sigma')$. Combining the first and last terms in the above expression, we get \mbox{$\x_{t,i}(\sigma') \geq \x_{t,i}(\sigma) -\frac{2\gamma(\sigma)}{c} - \beta$}.
\end{proof}
\begin{lemma}[Monotonicity 2]
Let $\x_t(\sigma)$ be the prediction of FTPL in iteration $t$, with random perturbation $\sigma$.
Let $\e_i$ denote the $i^{th}$ standard basis vector and $\x_{t,i}$ denote the $i^{th}$ coordinate of $\x_t$. Suppose \mbox{$\norm{\x_{t}(\sigma) - \x_{t+1}(\sigma)}_1 \leq 10d \cdot |\x_{t,i}(\sigma) - \x_{t+1,i}(\sigma)|$.}   For $\sigma' = \sigma+100Ld \e_i$, we have
    \label{lem:monotone2}
\begin{align*}
    \min\left(\x_{t,i}(\sigma'), \x_{t+1, i}(\sigma')\right) \geq & \ \max\left(\x_{t,i}(\sigma), \x_{t+1,i}(\sigma)\right) - \frac{1}{10}|\x_{t,i}(\sigma)-\x_{t+1,i}(\sigma)|\\
    & - \frac{3(\alpha+\beta\|\sigma\|_{1})}{100Ld} -\beta.
\end{align*}
\end{lemma}
\begin{proof}
Let $f_{1:t}(\x) = \sum_{i = 1}^tf_i(\x)$ and let $\gamma(\sigma) = \alpha + \beta \|\sigma\|_{1}$ be the approximation error of the offline optimization oracle. From the approximate optimality of $\x_t(\sigma)$, we have
\begingroup\makeatletter\def\f@size{10}\check@mathfonts
\begin{align*}
			&f_{1:t-1}(\x_t(\sigma)) - \iprod{\sigma}{\x_t(\sigma)}  + f_t(\x_t(\sigma)) \\
			&\quad \leq f_{1:t-1}(\x_{t+1}(\sigma)) - \iprod{\sigma}{\x_{t+1}(\sigma)} + f_t(\x_t(\sigma)) + \gamma(\sigma)\\
			&\quad \stackrel{(a)}{\leq} f_{1:t-1}(\x_{t+1}(\sigma)) - \iprod{\sigma}{\x_{t+1}(\sigma)} + f_t(\x_{t+1}(\sigma)) + L\norm{\x_{t}(\sigma)-\x_{t+1}(\sigma)}_1 + \gamma(\sigma)\\
			&\quad \stackrel{(b)}{\leq} f_{1:t-1}(\x_{t+1}(\sigma)) - \iprod{\sigma}{\x_{t+1}(\sigma)} + f_t(\x_{t+1}(\sigma)) + 10Ld|\x_{t,i}(\sigma) - \x_{t+1,i}(\sigma)| + \gamma(\sigma),
\end{align*}
\endgroup
where $(a)$ follows from the Lipschitz property of $f_t(\cdot)$ and $(b)$ follows from our assumption on $\norm{\x_{t}(\sigma) - \x_{t+1}(\sigma)}_1$.
Next, from the optimality of $\x_{t+1}(\sigma')$, we have
\begingroup\makeatletter\def\f@size{10}\check@mathfonts
\begin{align*}
			&f_{1:t-1}(\x_t(\sigma)) - \iprod{\sigma}{\x_t(\sigma)}  + f_t(\x_t(\sigma)) \\
			&\quad = f_{1:t-1}(\x_t(\sigma)) - \iprod{\sigma'}{\x_t(\sigma)}  + f_t(\x_t(\sigma)) + \iprod{100Ld\e_i}{\x_t(\sigma)}\\
			&\quad \geq f_{1:t-1}(\x_{t+1}(\sigma')) - \iprod{\sigma'}{\x_{t+1}(\sigma')} + f_t(\x_{t+1}(\sigma')) + 100Ld\x_{t,i}(\sigma)-\gamma(\sigma')\\
			&\quad = f_{1:t-1}(\x_{t+1}(\sigma')) - \iprod{\sigma}{\x_{t+1}(\sigma')} + f_t(\x_{t+1}(\sigma')) + 100Ld(\x_{t,i}(\sigma)-\x_{t+1,i}(\sigma'))-\gamma(\sigma')\\
			&\quad \geq f_{1:t-1}(\x_{t+1}(\sigma)) - \iprod{\sigma}{\x_{t+1}(\sigma)} + f_t(\x_{t+1}(\sigma)) +100Ld(\x_{t,i}(\sigma)-\x_{t+1,i}(\sigma')) - \gamma(\sigma') - \gamma(\sigma),
\end{align*}
\endgroup
where the last inequality follows from the optimality of $\x_{t+1}(\sigma)$. 
Combining the above two equations, we get
\[
\x_{t+1,i}(\sigma')-\x_{t,i}(\sigma) \geq -\frac{1}{10}|\x_{t,i}(\sigma) - \x_{t+1,i}(\sigma)|-\frac{3\gamma(\sigma)}{100Ld} - \beta.
\]
A similar argument shows that 
\[
\x_{t,i}(\sigma')-\x_{t+1,i}(\sigma) \geq -\frac{1}{10}|\x_{t,i}(\sigma) - \x_{t+1,i}(\sigma)|-\frac{3\gamma(\sigma)}{100Ld} - \beta.
\]
Finally, from the monotonicity property in Lemma~\ref{lem:monotone1} we know that
\[
\x_{t+1,i}(\sigma')- \x_{t+1,i}(\sigma) \geq -\frac{3\gamma(\sigma)}{100Ld}-\beta,\quad \x_{t,i}(\sigma')- \x_{t,i}(\sigma) \geq -\frac{3\gamma(\sigma)}{100Ld}-\beta.
\]
Combining the above four inequalities gives us the required result.
\end{proof}
\paragraph{Proof of Theorem~\ref{thm:ftpl}.} We now proceed to the proof of Theorem~\ref{thm:ftpl}. We use the same notation as in Lemmas~\ref{lem:monotone1},~\ref{lem:monotone2}.
First note that $\E{\|\x_t(\sigma)-\x_{t+1}(\sigma)\|_1}$ can be written as
\begin{equation}
    \label{eqn:coordinatewise}
    \E{\norm{\x_t(\sigma)-\x_{t+1}(\sigma)}_1} = \sum_{i = 1}^d \E{|\x_{t,i}(\sigma)-\x_{t+1,i}(\sigma)|}.
\end{equation}
To bound $\E{\norm{\x_t(\sigma)-\x_{t+1}(\sigma)}_1}$ we derive an upper bound for \mbox{$\E{|\x_{t,i}(\sigma) - \x_{t+1,i}(\sigma)|}, \forall i\in [d]$}. For any $i \in [d]$, define $\Eover{-i}{|\x_{t,i}(\sigma) - \x_{t+1,i}(\sigma)|}$ as 
$$\Eover{-i}{|\x_{t,i}(\sigma) - \x_{t+1,i}(\sigma)|} \coloneqq \E{|\x_{t,i}(\sigma) - \x_{t+1,i}(\sigma)|\Big| \{\sigma_{j}\}_{j\neq i}},$$
where $\sigma_j$ is the $j^{th}$ coordinate of $\sigma$.  Let $\x_{max,i}(\sigma) = \max\left(\x_{t,i}(\sigma), \x_{t+1,i}(\sigma)\right)$ and $\x_{min,i}(\sigma) = \min\left(\x_{t,i}(\sigma), \x_{t+1,i}(\sigma)\right)$. Then $\Eover{-i}{|\x_{t,i}(\sigma) - \x_{t+1,i}(\sigma)|} = \Eover{-i}{\x_{max,i}(\sigma)} - \Eover{-i}{\x_{min,i}(\sigma)}$. 
Define  event $\mathcal{E}$ as 
$$\mathcal{E} = \left\lbrace \sigma :\norm{\x_{t}(\sigma)-\x_{t+1}(\sigma)}_1 \leq {10  d} \cdot \abs{\x_{t,i}(\sigma)-\x_{t+1,i}(\sigma)}\right\rbrace.$$ 
Consider the following
\begin{equation*}
    \begin{array}{lll}
         \Eover{-i}{\x_{min,i}(\sigma)}  &=& \Pr(\sigma_i < 100Ld) \Eover{-i}{\x_{min,i}(\sigma)|\sigma_i < 100Ld} \vspace{0.1in}\\
         && + \Pr(\sigma_i \geq 100Ld) \Eover{-i}{\x_{min,i}(\sigma) |\sigma_i \geq 100Ld}\vspace{0.1in}\\
         &\geq& \left(1-\exp(-100\eta Ld)\right)(\Eover{-i}{\x_{max, i}(\sigma)}-D) \vspace{0.1in}\\
         && + \exp(-100\eta Ld)\Eover{-i}{\x_{min,i}(\sigma+100Ld\e_i)},
\end{array}
\end{equation*}
where the last inequality follows  the fact that the domain of $i^{th}$ coordinate lies within some interval of length $D$ and since $\Eover{-i}{\x_{min,i}(\sigma)|\sigma_i < 100Ld}$ and $\Eover{-i}{\x_{max, i}(\sigma)}$ are points in this interval, their difference is bounded by $D$. We can further lower bound $\Eover{-i}{\x_{min,i}(\sigma)}$ as follows
\begin{equation*}
    \begin{array}{lll}
         \Eover{-i}{\x_{min,i}(\sigma)} 
         &\geq & \left(1-\exp(-100\eta Ld)\right)(\Eover{-i}{\x_{max, i}(\sigma)}-D)\vspace{0.1in}\\
         && + \exp(-100\eta Ld)\Pr_{-i}(\mathcal{E})\Eover{-i}{\x_{min,i}(\sigma+100Ld\e_i) | \mathcal{E}} \vspace{0.1in}\\
         && + \exp(-100\eta Ld)\Pr_{-i}(\mathcal{E}^c)\Eover{-i}{\x_{min,i}(\sigma+100Ld\e_i) | \mathcal{E}^c},
\end{array}
\end{equation*}
where $\mathbb{P}_{-i}(\mathcal{E})$ is defined as $\Pr_{-i}(\mathcal{E}) \coloneqq \Pr\left(\mathcal{E}\Big| \{\sigma_{j}\}_{j\neq i}\right).$ We now use the monotonicity properties proved in Lemmas~\ref{lem:monotone1},~\ref{lem:monotone2} to further lower bound $\Eover{-i}{\x_{min,i}(\sigma)}$. Let $\gamma(\sigma) = \alpha + \beta \|\sigma\|_{1}$ be the approximation error of the offline optimization oracle. Then
\begingroup\makeatletter\def\f@size{10}\check@mathfonts
\begin{equation*}
    \begin{array}{lll}
         \Eover{-i}{\x_{min,i}(\sigma)} &\geq&  \left(1-\exp(-100\eta Ld)\right)(\Eover{-i}{\x_{max, i}(\sigma)}-D)\vspace{0.1in}\\
         && + \exp(-100\eta Ld)\Pr_{-i}(\mathcal{E})\Eover{-i}{\x_{max,i}(\sigma) -\frac{1}{10}|\x_{t,i}(\sigma) - \x_{t+1,i}(\sigma)| - \frac{3\gamma(\sigma)}{100Ld}-\beta \Big| \mathcal{E}}\vspace{0.1in}\\
         && + \exp(-100\eta Ld)\Pr_{-i}(\mathcal{E}^c)\Eover{-i}{\x_{min,i}(\sigma) - \frac{2\gamma(\sigma)}{100Ld}-\beta| \mathcal{E}^c}\vspace{0.1in}\\
         &\geq& \left(1-\exp(-100\eta Ld)\right)(\Eover{-i}{\x_{max, i}(\sigma)}-D)\vspace{0.1in}\\
         && + \exp(-100\eta Ld)\Pr_{-i}(\mathcal{E})\Eover{-i}{\x_{max,i}(\sigma) -\frac{1}{10}|\x_{t,i}(\sigma) - \x_{t+1,i}(\sigma)|-\frac{3\gamma(\sigma)}{100Ld} -\beta\Big| \mathcal{E}}\vspace{0.1in}\\
         && + \exp(-100\eta Ld)\Pr_{-i}(\mathcal{E}^c)\Eover{-i}{\x_{max,i}(\sigma)-\frac{1}{10d}\|\x_{t}(\sigma)-\x_{t+1}(\sigma)\|_1 - \frac{2\gamma(\sigma)}{100Ld} -\beta\Big| \mathcal{E}^c},
    \end{array}
\end{equation*}
\endgroup
where the first inequality follows from Lemmas~\ref{lem:monotone1},~\ref{lem:monotone2}, the second inequality follows from the definition of $\mathcal{E}^c$. Rearranging the terms in the RHS and using $\Pr_{-i}(\mathcal{E}) \leq 1$ gives us 
\begin{equation*}
    \begin{array}{lll}
         \Eover{-i}{\x_{min,i}(\sigma)} &\geq&   \left(1-\exp(-100\eta Ld)\right)(\Eover{-i}{\x_{max, i}(\sigma)}-D)\vspace{0.1in}\\
         && + \exp(-100\eta Ld)\Eover{-i}{\x_{max,i}(\sigma) - \frac{3\gamma(\sigma)}{100Ld} - \beta}\vspace{0.1in}\\
         && - \exp(-100\eta Ld)\Eover{-i}{\frac{1}{10}|\x_{t,i}(\sigma) - \x_{t+1,i}(\sigma)|+\frac{1}{10d}\|\x_{t}(\sigma) - \x_{t+1}(\sigma)\|_1}\vspace{0.1in}\\
         &\geq& \Eover{-i}{\x_{max, i}(\sigma)}  -100\eta L d D - \frac{3\gamma(\sigma)}{100Ld}-\beta\vspace{0.1in}\\
         && - \Eover{-i}{\frac{1}{10}|\x_{t,i}(\sigma) - \x_{t+1,i}(\sigma)|+\frac{1}{10d}\|\x_{t}(\sigma) - \x_{t+1}(\sigma)\|_1},
    \end{array}
\end{equation*}
where the last inequality uses the the fact that $\exp(x) \geq 1 + x$. Rearranging the terms in the last inequality gives us
\begin{align*}
    \Eover{-i}{|\x_{t,i}(\sigma) - \x_{t+1,i}(\sigma)|} &\leq \frac{1}{9d}\Eover{-i}{\norm{\x_{t}(\sigma) - \x_{t+1}(\sigma)}_1}\\
    &\quad + \frac{1000}{9}\eta L  d D  + \frac{\Eover{-i}{\gamma(\sigma)}}{30Ld} + \frac{10}{9}\beta.
\end{align*}
Since the above bound holds for any $\{\sigma_j\}_{j\neq i}$, we get the following bound on the unconditioned expectation
\begin{align*}
\E{|\x_{t,i}(\sigma) - \x_{t+1,i}(\sigma)|}  \leq & \frac{1}{9d}\E{\norm{\x_{t}(\sigma) - \x_{t+1}(\sigma)}_1}\\
 & + \frac{1000}{9}\eta L d D + \frac{\E{\gamma(\sigma)}}{30Ld} + \frac{10}{9}\beta.
\end{align*}
Plugging this in Equation~\eqref{eqn:coordinatewise} gives us the following bound on stability of predictions of FTPL
\[
\E{\norm{\x_{t}(\sigma) - \x_{t+1}(\sigma)}_1} \leq 125\eta L d^2 D + \frac{\beta d}{20 \eta L} + 2\beta d + \frac{\alpha}{20L}.
\]
Plugging the above bound in Equation~\eqref{eqn:stability} gives us the required bound on regret.

\section{Non-Convex OFTPL}
\label{sec:oftpl}
 In this section, we present a proof of Theorem~\ref{thm:oftpl}. Since we are in the oblivious adversary model, similar to the proof of Theorem~\ref{thm:ftpl}, we work with a single random vector $\sigma$ over the entire algorithm. We first relate the expected regret of OFTPL to the stability of its prediction. Unlike Lemma~\ref{lem:reg_to_stbl}, the upper bound we obtain for OFTPL depends on the Lipschitz constant of $(f_t-g_t)$. 
\begin{lemma}
    \label{lem:reg_to_stbl_oftpl}
    Let $\bx_t$ be any minimizer of $\sum_{i=1}^{t-1}f_i(\x) - \iprod{\sigma}{\x}$.
The regret of OFTPL can be upper bounded as 
    \begin{equation}
        \label{eqn:stability_oftpl}
        \E{\sum_{t = 1}^T f_t(\x_t) - \inf_{\x \in \mathcal{X}}\sum_{t=1}^Tf_t(\x)} \leq \sum_{t=1}^TL_t\E{\norm{\x_t-\bx_{t+1}}_1} +  \frac{d(\beta T+D)}{\eta} + \alpha T.
    \end{equation}
\end{lemma}
The rest of the proof of Theorem~\ref{thm:oftpl} involves bounding $\E{\norm{\x_t-\bx_{t+1}}_1}$ and uses identical arguments as in the proof of Theorem~\ref{thm:ftpl} (see Appendix~\ref{sec:apx_oftpl_proof}).

\section{Conclusion}
\label{sec:conclusion}
In this work, we considered the  problem of online learning with non-convex losses and showed that the classical FTPL algorithm with access to an offline optimization oracle achieves optimal regret rate of $O(T^{-1/2})$. We further showed that an optimistic variant of FTPL can achieve better regret bounds when the sequence of losses are predictable. 

The problem of online non-convex learning has several important applications in machine learning. We believe the algorithms studied in this work can lead to improved training procedures for adversarial training and training of Generative Adversarial Networks, which currently rely on algorithms from online convex learning to solve the non-convex non-concave saddle point problems in their training objectives.

\appendix
\section{Proof of Proposition~\ref{prop:no_deterministic}}
For any deterministic algorithm, we show that there exists a sequence of loss functions over which the algorithm has $\Omega(T)$ regret. We work in the $1$-dimensional setting and assume that the domain $\mathcal{X}$ is equal to $[-D, D]$. Suppose the adversary chooses the loss functions from the following class of $1$-Lipschitz functions $\mathcal{F} = \{g_a(\x): a \in [-D, D]\}$, where $g_a$ is given by $$g_a(\x) = \max\left\lbrace 0, \frac{D}{2} - |\x - a|\right\rbrace.$$ 
We now describe our construction of the sequence of losses that cause the deterministic algorithm to fail. Let $f_{< t} = \{f_1, \dots f_{t-1}\}$ be the sequence of loss functions chosen until iteration $t-1$. Let $\x_t$ be the prediction of the deterministic learner at iteration $t$. Then we choose the loss at iteration $t$ as $f_t(\x) = g_{\x_t}(\x)$. It is easy to see that, after $T$ iterations, the loss suffered by the learner is equal to $\frac{DT}{2}$. Whereas, the loss of the best action in hindsight can be upper bounded as
\[
\inf_{\x \in [-D,D]} \sum_{t=1}^T f_t(\x) \leq \frac{DT}{4}.
\]
This shows that the regret of any deterministic algorithm is $\Omega(1)$.
\section{Non-oblivious to Oblivious Adversary Model}
\label{sec:apx_reduction}
In the oblivious adversary model, the actions $\{f_t\}_{t=1}^T$ of the adversary are assumed to be independent of the predictions $\{\x_t\}_{t=1}^T$ of the FTPL/OFTPL algorithm. In this model, we assume that the sequence of losses $\{f_t\}_{t=1}^T$ is fixed ahead of time. Whereas in the non-oblivious adversary model, the actions of the adversary are allowed to depend on the past predictions of the algorithm, \emph{i.e.,} each $f_t$ is given by $f_t \coloneqq F_{t}[\x_{<t}]$ for some function $F_t:\mathcal{X}^{t-1}\rightarrow \mathcal{F}$, where $\mathcal{F}$ is the set of all possible actions of the adversary and $\x_{<t}$ is a shorthand for $\{\x_1\dots \x_{t-1}\}$ and $F_1$ is a constant function. Note that the functions $F_1\dots F_T$ uniquely determine a non-oblivious adversary. 

Let $P_t$ be the conditional distribution of the prediction $\x_t$ of the FTPL/OFTPL algorithm, conditioned on the past predictions $\x_{<t}$. Note that when the adversary is oblivious, $P_t$ is independent of $\x_{<t}$. Moreover, in both oblivious and non-oblivious models, $P_t$ is fully determined by the past actions $f_{<t}$ of the adversary. Let $f_t(P_t)$ denote the expected loss $\Eover{\x\sim P_t}{f_t(\x)|\x_{<t}}$. 

The following Theorem  shows that any algorithm which is guaranteed to work against an oblivious adversary also works against a non-oblivious adversary. This is an adaptation of Lemma 4.1 of \cite{cesa2006prediction} to the setting studied in this paper.
\begin{theorem}
\label{thm:reduction_oblv}
Let $B$ be a positive constant. Suppose the FTPL, OFTPL algorithms satisfy the following regret bound against an oblivious adversary
\begin{equation}
\label{eqn:oblivious_regret}
    \E{\sum_{t=1}^T f_t(\x_t) - \inf_{\x \in \mathcal{X}}\sum_{t=1}^T f_t(\x)} \leq B,\quad  \forall f_1\dots f_T \in \mathcal{F}.
\end{equation}
Then these algorithms satisfy the following regret bound against a non-oblivious adversary
\[
\sum_{t=1}^Tf_t(P_t) - \inf_{\x \in \mathcal{X}}\sum_{t=1}^T f_t(\x) \leq B.
\]
\end{theorem}
\begin{proof}
Consider the non-oblivious adversary model. For any $\x \in \mathcal{X}$ we have
\begin{align*}
    &\sum_{t=1}^Tf_t(P_t) - \sum_{t=1}^T f_t(\x) \\
    &\quad = \sum_{t=1}^TF_t[\x_{<t}](P_t) - \sum_{t=1}^T F_t[\x_{<t}](\x)\\
    &\quad \stackrel{(a)}{\leq} \sup_{F_1,\dots F_T} \left(\sum_{t=1}^TF_t[\x_{<t}](P_t) - \sum_{t=1}^T F_t[\x_{<t}](\x)\right)\\
    &\quad \stackrel{(b)}{=} \sup_{g_1 \in \mathcal{F}} \left(g_1(P_1) - g_1(\x) + \sup_{g_2 \in \mathcal{F}}\left(g_2(P_2) - g_2(\x) + \sup_{g_3\in\mathcal{F}}\left(\dots + \sup_{g_T\in \mathcal{F}} g_T(P_T) - g_T(\x)\right) \right)\right),
\end{align*}
where the supremum in $(a)$ is over all possible non-oblivious adversaries. To see why $(b)$ holds, consider $T = 2$. Then
\begin{align*}
    &\sup_{F_1,F_2} \left(F_1[\x_{<1}](P_1) -  F_1[\x_{<1}](\x) + F_2[\x_{<2}](P_2) -  F_2[\x_{<2}](\x)\right)\\
    & \quad = \sup_{g_1 \in \mathcal{F},F_2} \left(g_1(P_1) -  g_1(\x) + F_2[\x_{<2}](P_2) -  F_2[\x_{<2}](\x)\right)\\
    &\quad = \sup_{g_1} \left(g_1(P_1) -  g_1(\x) + \sup_{g_2 \in \mathcal{F}}g_2(P_2) -  g_2(\x)\right).
\end{align*}
This shows that a good strategy for the adversary is to set $F_2[\x_{<2}]$ to be a maximizer of $g_2(P_2) - g_2(\x)$. 
Using a similar argument we can show that $(b)$ holds for $T  > 2$. 

Next, we show that 
\begin{align*}
    &\sup_{g_1 \in \mathcal{F}} \left(g_1(P_1) - g_1(\x) + \sup_{g_2 \in \mathcal{F}}\left(g_2(P_2) - g_2(\x) + \sup_{g_3\in\mathcal{F}}\left(\dots + \sup_{g_T\in \mathcal{F}} g_T(P_T) - g_T(\x)\right) \right)\right)\\
    &\quad = \sup_{g_1\dots g_T \in \mathcal{F}} \left(\sum_{t=1}^Tg_t(P_t) - g_t(\x)\right).
\end{align*}
Moreover, we show that the maximizers of the RHS objective are independent of the predictions $\{\x_t\}_{t=1}^T$ of the algorithm. This would then imply that the RHS is exactly equal to the regret of the algorithm under the oblivious adversary model, which is upper bounded by $B$.
To see why the above statements are true, again consider the case of $T=2$. First note that $g_1(P_1) - g_1(\x)$ is independent of $g_2$.
So $g_1(P_1) - g_1(\x)$ can be pushed inside the inner supermum. So we have
\begin{align*}
    & \sup_{g_1\in \mathcal{F}} \left(g_1(P_1) - g_1(\x) + \sup_{g_2 \in \mathcal{F}} \left(g_2(P_2) - g_2(\x)\right)\right)\\
    &\quad = \sup_{g_1, g_2\in \mathcal{F}} \left(g_1(P_1) - g_1(\x) + g_2(P_2) - g_2(\x)\right)
\end{align*}
To see why the maximizers of the RHS are independent of $\x_1,\x_2$, note that $P_1$ is independent of $\x_1,\x_2$. Moreover, $P_2$ is fully determinimed by $g_1$. So the objective is independent of $\x_1,\x_2$. This shows that the maximizers are independent of $\x_1,\x_2$. Using a similar argument we can show that the above claim holds for $T>2$. Finally, from the regret bound against an oblivious adversary in Equation~\eqref{eqn:oblivious_regret}, we have
\[
\sup_{g_1\dots g_T \in \mathcal{F}} \left(\sum_{t=1}^Tg_t(P_t) - g_t(\x)\right) = \sup_{g_1\dots g_T \in \mathcal{F}} \E{\sum_{t = 1}^T g_t(\x_t) - \sum_{t=1}^T g_t(\x)} \leq B.
\]
This shows that for any $\x \in \mathcal{X}$, $\sum_{t=1}^Tf_t(P_t) - \sum_{t=1}^T f_t(\x) \leq B$.
\end{proof}
\section{Proof of Lemma~\ref{lem:reg_to_stbl}}
\label{sec:apx_stbl_ftpl}
Let $\gamma(\sigma) = \alpha + \beta \|\sigma\|_{1}$. For any $\x^*\in \mathcal{X}$ we have
\begin{align*}
    &\sum_{t = 1}^T \left[f_t(\x_t) - f_t(\x^*)\right]  \\
    &\quad = \sum_{t = 1}^T \left[f_t(\x_t) -f_t(\x_{t+1})\right] + \sum_{t = 1}^T \left[ f_t(\x_{t+1}) - f_t(\x^*) \right]\\
    & \quad \leq \sum_{t=1}^TL\norm{\x_t - \x_{t+1}}_1 + \sum_{t=1}^T\left[f_t(\x_{t+1}) - f_t(\x^*)\right].
\end{align*}
We now use induction to show that $\sum_{t=1}^T\left[f_t(\x_{t+1}) - f_t(\x^*)\right] \leq \gamma(\sigma) T + \iprod{\sigma}{\x_2-\x^*}$.
\paragraph{Base Case ($T=1$).} Since $\x_2$ is an approximate minimizer of $f_1(\x) - \iprod{\sigma}{\x}$, we have
\[
f_1(\x_2) - \iprod{\sigma}{\x_2} \leq \min_{\x \in \mathcal{X}} f_1(\x) - \iprod{\sigma}{\x}  + \gamma(\sigma) \leq f_1(\x^*) - \iprod{\sigma}{\x^*}  + \gamma(\sigma),
\]
where the last inequality holds for any $\x^* \in \mathcal{X}$. This shows that \mbox{$f_1(\x_2)-f_1(\x^*) \leq \gamma(\sigma) + \iprod{\sigma}{\x_2-\x^*}$.}
\paragraph{Induction Step.} Suppose the claim holds for all $T \leq T_0-1$. We now show that it also holds for $T_0$.  
\begin{align*}
    & \sum_{t=1}^{T_0}f_t(\x_{t+1})\\
    &\quad \stackrel{(a)}{\leq} \left[\sum_{t=1}^{T_0-1}f_t(\x_{T_0+1}) + \iprod{\sigma}{\x_2-\x_{T_0+1}} + \gamma(\sigma) (T_0-1) \right]+ f_{T_0}(\x_{T_0+1})\\
    &\quad = \left[\sum_{t=1}^{T_0}f_t(\x_{T_0+1}) - \iprod{\sigma}{\x_{T_0+1}}\right] + \iprod{\sigma}{\x_{2}} + \gamma(\sigma) (T_0-1)\\
    &\quad \stackrel{(b)}{\leq} \sum_{t=1}^{T_0}f_t(\x^*) + \iprod{\sigma}{\x_{2}-\x^*} + \gamma(\sigma) T_0, \quad \forall \x^* \in \mathcal{X},
\end{align*}
where $(a)$ follows since the claim holds for any $T\leq T_0-1$, and $(b)$ follows from the approximate optimality of $\x_{T_0+1}$.

Using this result, we get the following upper bound on the expected regret of FTPL
\begin{align*}
    \E{\sum_{t = 1}^T f_t(\x_t) - \inf_{\x \in \mathcal{X}}\sum_{t=1}^Tf_t(\x)} &\leq L\sum_{t=1}^T\E{\norm{\x_t-\x_{t+1}}_1} + \E{\gamma(\sigma) T + \iprod{\sigma}{\x_2-\x^*}}\\
    & \leq L\sum_{t=1}^T\E{\norm{\x_t-\x_{t+1}}_1} + (\beta T+D)\left(\sum_{i=1}^d \E{\sigma_i}\right) + \alpha T
\end{align*}
The proof of the Lemma now follows from the following property of exponential distribution 
$$\E{\sigma_i} = \frac{1}{\eta_i}.$$
\section{Proof of Lemma~\ref{lem:reg_to_stbl_oftpl}}
\label{sec:apx_stbl_oftpl}
The proof uses similar arguments as in the proof of \citet{rakhlin2012online} for Optimistic FTRL. Let $\Delta_t(\x) = f_t(\x)-g_t(\x)$ and $\gamma(\sigma) = \alpha + \beta \|\sigma\|_{ 1}$. For any $\x^*\in \mathcal{X}$ we have
\begin{align*}
    &\sum_{t = 1}^T \left[f_t(\x_t) - f_t(\x^*)\right]  \\
    & \quad = \sum_{t=1}^T \left[\Delta_t(\x_t) - \Delta_t(\bx_{t+1})\right] + \sum_{t=1}^T\left[g_t(\x_t) - g_t(\bx_{t+1})\right] + \sum_{t=1}^T\left[f_t(\bx_{t+1}) - f_t(\x^*)\right]\\
    & \quad \leq \sum_{t=1}^TL_t\norm{\x_t - \bx_{t+1}}_1 + \sum_{t=1}^T\left[g_t(\x_t) - g_t(\bx_{t+1})\right] + \sum_{t=1}^T\left[f_t(\bx_{t+1}) - f_t(\x^*)\right].
\end{align*}
We use induction to show that the following holds for any $T,\x^*\in \mathcal{X}$
$$\sum_{t=1}^T\left[g_t(\x_t) - g_t(\bx_{t+1})\right] + \sum_{t=1}^T\left[f_t(\bx_{t+1}) - f_t(\x^*)\right] \leq \iprod{\sigma}{\bx_{2}-\x^*} + \gamma(\sigma) (T-1).$$
\paragraph{Base Case ($T=1$).} First note that $g_1 = 0$. Since $\bx_2$ is a minimizer of $f_1(\x) - \iprod{\sigma}{\x}$, we have
\[
f_1(\bx_2) - \iprod{\sigma}{\bx_2} \leq  f_1(\x^*) - \iprod{\sigma}{\x^*},\quad \forall \x^* \in \mathcal{X}.
\]
This shows that $f_1(\bx_2)-f_1(\x^*) \leq \iprod{\sigma}{\bx_2-\x^*}$.
\paragraph{Induction Step.} Suppose the claim holds for all $T \leq T_0-1$. We now show that it also holds for $T_0$. Consider the following series of inequalities  
\begin{align*}
    & \sum_{t=1}^{T_0}\left[g_t(\x_t) - g_t(\bx_{t+1})\right] + \sum_{t=1}^{T_0}f_t(\bx_{t+1}) \\
    & \quad \stackrel{(a)}{\leq} \left[\sum_{t=1}^{T_0-1}f_t(\x_{T_0}) + \iprod{\sigma}{\bx_{2} - \x_{T_0}} +\gamma(\sigma)(T_0-2)\right]+ \left[g_{T_0}(\x_{T_0}) - g_{T_0}(\bx_{T_0+1})+ f_{T_0}(\bx_{T_0+1})\right]\\
    &\quad =  \left[\sum_{t=1}^{T_0-1}f_t(\x_{T_0})+g_{T_0}(\x_{T_0}) - \iprod{\sigma}{\x_{T_0}}\right] + \left[\iprod{\sigma}{\bx_{2}} - g_{T_0}(\bx_{T_0+1})+ f_{T_0}(\bx_{T_0+1})\right]+\gamma(\sigma)(T_0-2)\\
    & \quad \stackrel{(b)}{\leq} \left[\sum_{t=1}^{T_0-1}f_t(\bx_{T_0+1})+g_{T_0}(\bx_{T_0+1}) - \iprod{\sigma}{\bx_{T_0+1}}\right] + \left[\iprod{\sigma}{\bx_{2}} - g_{T_0}(\bx_{T_0+1})+ f_{T_0}(\bx_{T_0+1})\right] +\gamma(\sigma)(T_0-1)\\
    & \quad = \left[\sum_{t=1}^{T_0}f_t(\bx_{T_0+1}) - \iprod{\sigma}{\bx_{T_0+1}}\right] +\iprod{\sigma}{\bx_{2}}  +\gamma(\sigma)(T_0-1)\\
    & \quad \stackrel{(c)}{\leq} \sum_{t=1}^{T_0}f_t(\x^*) + \iprod{\sigma}{\bx_{2}-\x^*} + \gamma(\sigma)(T_0-1),
\end{align*}
where $(a)$ follows since the claim holds for any $T\leq T_0-1$, $(b)$ follows from the approximate optimality of $\x_{T_0}$ and $(c)$ follows from the optimality of $\bx_{T_0+1}$.

This gives the following upper bound on the regret of OFTPL
\[
\sum_{t = 1}^T f_t(\x_t) - \inf_{\x\in\mathcal{X}}\sum_{i=1}^{T}f_t(\x) \leq \sum_{t=1}^TL_t\norm{\x_t - \bx_{t+1}}_1 + \frac{d(\beta T+D)}{\eta} + \alpha T .
\]

\section{Proof of Theorem~\ref{thm:oftpl}}
\label{sec:apx_oftpl_proof}
\begin{lemma}
 Let $\x_t(\sigma)$ be the prediction of OFTPL in iteration $t$, with random perturbation $\sigma$. Then the following monotonicity property holds for any $c> 0$
\[
\x_{t,i}(\sigma+c\e_i) \geq \x_{t,i}(\sigma) - \frac{2(\alpha+\beta\|\sigma\|_{1})}{c} - \beta.
\]
\end{lemma}
\begin{proof}
Let $f_{1:t}(\x) = \sum_{i = 1}^tf_i(\x)$ and $\sigma' = \sigma+c \e_i$. Moreover, let $\gamma(\sigma) = \alpha + \beta \|\sigma\|_{1}$ be the approximation error of the offline optimization oracle. From the approximate optimality of $\x_t(\sigma)$ we have
\begin{align*}
&f_{1:t-1}(\x_t(\sigma)) + g_{t}(\x_t(\sigma)) - \iprod{\sigma}{\x_t(\sigma)}\\
&\quad \leq f_{1:t-1}(\x_t(\sigma'))+ g_{t}(\x_t(\sigma')) - \iprod{\sigma}{\x_t(\sigma')} + \gamma(\sigma)\\
&\quad = f_{1:t-1}(\x_t(\sigma'))+ g_{t}(\x_t(\sigma')) - \iprod{\sigma'}{\x_t(\sigma')} + c\x_{t,i}(\sigma') + \gamma(\sigma)\\
&\quad \stackrel{(a)}{\leq} f_{1:t-1}(\x_t(\sigma))+ g_{t}(\x_t(\sigma)) - \iprod{\sigma'}{\x_t(\sigma)} + c\x_{t,i}(\sigma') + \gamma(\sigma) + \gamma(\sigma')\\
&\quad = f_{1:t-1}(\x_t(\sigma))+ g_{t}(\x_t(\sigma)) - \iprod{\sigma}{\x_t(\sigma)} + c\left(\x_{t,i}(\sigma') - \x_{t,i}(\sigma)\right) + \gamma(\sigma) + \gamma(\sigma'),
\end{align*}
where $(a)$ follows from the approximate optimality of $\x_t(\sigma')$. Combining the first and last terms in the above expression, we get \mbox{$\x_{t,i}(\sigma') \geq \x_{t,i}(\sigma) -\frac{2\gamma(\sigma)}{c} - \beta$}.
\end{proof}
We note that a similar argument can be used to show that $\bx_{t,i}(\alpha + c\e_i) \geq \bx_{t,i}(\alpha)$.
\begin{lemma}
 Suppose \mbox{$\norm{\x_{t}(\sigma) - \bx_{t+1}(\sigma)}_1 \leq 10d \cdot |\x_{t,i}(\sigma) - \bx_{t+1,i}(\sigma)|$.}   For $\sigma' = \sigma+100L_td \e_i$, we have
\begin{align*}
    \min\left(\x_{t,i}(\sigma'), \bx_{t+1, i}(\sigma')\right) \geq & \ \max\left(\x_{t,i}(\sigma), \bx_{t+1,i}(\sigma)\right) - \frac{1}{10}|\x_{t,i}(\sigma)-\bx_{t+1,i}(\sigma)|\\
    & - \frac{3(\alpha+\beta\|\sigma\|_{1})}{100L_td} -\beta.
\end{align*}
\end{lemma}
\begin{proof}
Let $f_{1:t}(\x) = \sum_{i = 1}^tf_i(\x)$ and let $\gamma(\sigma) = \alpha + \beta \|\sigma\|_{1}$ be the approximation error of the offline optimization oracle. From the approximate optimality of $\x_t(\sigma)$, we have
\begin{align*}
			&f_{1:t-1}(\x_t(\sigma)) - \iprod{\sigma}{\x_t(\sigma)}  + f_t(\x_t(\sigma)) \\
			&\quad \leq f_{1:t-1}(\bx_{t+1}(\sigma))+ g_{t}(\bx_{t+1}(\sigma)) - \iprod{\sigma}{\bx_{t+1}(\sigma)}\vspace{0.1in}\\
			& \quad \quad + f_t(\x_t(\sigma))- g_{t}(\x_t(\sigma)) + \gamma(\sigma)\\
			&\quad \stackrel{(a)}{\leq} f_{1:t-1}(\bx_{t+1}(\sigma))+ g_{t}(\bx_{t+1}(\sigma)) - \iprod{\sigma}{\bx_{t+1}(\sigma)} \vspace{0.1in}\\
			& \quad \quad + f_t(\bx_{t+1}(\sigma))- g_{t}(\bx_{t+1}(\sigma))  + L_t\norm{\x_{t}(\sigma)-\bx_{t+1}(\sigma)}_1 + \gamma(\sigma)\\
			&\quad \stackrel{(b)}{\leq} f_{1:t-1}(\bx_{t+1}(\sigma))+ g_{t}(\bx_{t+1}(\sigma)) - \iprod{\sigma}{\bx_{t+1}(\sigma)} \vspace{0.1in}\\
			&\quad \quad + f_t(\bx_{t+1}(\sigma))- g_{t}(\bx_{t+1}(\sigma)) + 10L_td|\x_{t,i}(\sigma) - \bx_{t+1,i}(\sigma)| + \gamma(\sigma),
\end{align*}
where $(a)$ follows from the Lipschitz property of $f_t(\cdot)$ and $(b)$ follows from our assumption on $\norm{\x_{t}(\sigma) - \bx_{t+1}(\sigma)}_1$.
Next, from the optimality of $\bx_{t+1}(\sigma')$, we have
\begingroup\makeatletter\def\f@size{10}\check@mathfonts
\begin{align*}
			&f_{1:t-1}(\x_t(\sigma)) - \iprod{\sigma}{\x_t(\sigma)}  + f_t(\x_t(\sigma)) \\
			&\quad = f_{1:t-1}(\x_t(\sigma)) - \iprod{\sigma'}{\x_t(\sigma)}  + f_t(\x_t(\sigma)) + \iprod{100L_td\e_i}{\x_t(\sigma)}\\
			&\quad \geq f_{1:t-1}(\bx_{t+1}(\sigma')) - \iprod{\sigma'}{\bx_{t+1}(\sigma')} + f_t(\bx_{t+1}(\sigma')) + 100L_td\x_{t,i}(\sigma)\\
			&\quad = f_{1:t-1}(\bx_{t+1}(\sigma')) - \iprod{\sigma}{\bx_{t+1}(\sigma')} + f_t(\bx_{t+1}(\sigma')) + 100L_td(\x_{t,i}(\sigma)-\bx_{t+1,i}(\sigma'))\\
			&\quad \geq f_{1:t-1}(\bx_{t+1}(\sigma)) - \iprod{\sigma}{\bx_{t+1}(\sigma)} + f_t(\bx_{t+1}(\sigma)) +100L_td(\x_{t,i}(\sigma)-\bx_{t+1,i}(\sigma')),
\end{align*}
\endgroup
where the last inequality follows from the optimality of $\bx_{t+1}(\sigma)$. 
Combining the above two equations, we get
\[
\bx_{t+1,i}(\sigma')-\x_{t,i}(\sigma) \geq -\frac{1}{10}|\x_{t,i}(\sigma) - \bx_{t+1,i}(\sigma)|-\frac{\gamma(\sigma)}{100L_td}.
\]
A similar argument shows that 
\[
\x_{t,i}(\sigma')-\bx_{t+1,i}(\sigma) \geq -\frac{1}{10}|\x_{t,i}(\sigma) - \bx_{t+1,i}(\sigma)|-\frac{\gamma(\sigma)}{100L_td}.
\]
Finally, from the monotonicity property in Lemma~\ref{lem:monotone1} we know that
\[
\bx_{t+1,i}(\sigma')- \bx_{t+1,i}(\sigma) \geq 0 ,\quad \x_{t,i}(\sigma')- \x_{t,i}(\sigma) \geq -\frac{3\gamma(\sigma)}{100L_td}-\beta.
\]
Combining the above four inequalities gives us the required result.
\end{proof}
\noindent The rest of the proof relies on the monotonicity properties showed in the above two Lemmas to bound $\E{\|\x_t-\bx_{t+1}\|_1}$ and uses identical arguments as in the proof of Theorem~\ref{thm:ftpl}. 

\bibliography{local}

\end{document}